\documentclass[oneside,12pt]{article}
\usepackage[utf8]{inputenc}
\usepackage[left = 1.2in, top = 1in, right = 1.2in]{geometry} % arxiv
\usepackage{setspace}
\usepackage{natbib}
\bibpunct{(}{)}{;}{a}{,}{,}
\usepackage{amsmath}
\usepackage{amssymb}
\usepackage{amsthm}
\usepackage{bm}
\usepackage[usenames,dvipsnames]{xcolor}
\usepackage{multirow}
\usepackage{psfrag}
\usepackage{tikz}
\usepackage{hyperref}
\hypersetup{colorlinks=true, urlcolor=Sepia, citecolor=Sepia, linkcolor=Sepia}
\usetikzlibrary{arrows,shapes,positioning,petri}
\def\m{{\bf m}}

\def\t{{\bf t}}

\def\w{{\bf w}}
\def\X{{\bf X}}
\def\x{{\bf x}}

\def\z{{\bf z}}

\def\0{{\bf 0}}
\def\1{{\bf 1}}

\def\<{\, \langle \,}
\def\>{\, \rangle \,}
\def\inv{^{-1}}
\def\ts{^\top}

\def\bpi{\bm{\pi}}

\def\bepsilon{\bm{\epsilon}}

\def\btheta{\bm{\theta}}

\def\bPhi{\bm{\Phi}}

\def\DN{\mathcal{N}}

\def\GIG{\mathrm{GIG}}

\def\Exp{\mathrm{Exponential}}

\def\Dis{\mathrm{Discrete}}

\newcommand{\argmin}{\operatornamewithlimits{argmin}}

\newcommand{\iid}{i.i.d.\ }

\def\ctilde{\kern -.04em\lower .7ex\hbox{\~{}}\kern .04em}

\newtheorem{lemma}{Lemma}
\begin{document}
\title{Efficient hierarchical clustering for continuous data
\footnotetext[0]{Ricardo Henao is Postdoctoral Associate and Joseph E. Lucas is Assistant Research Professor at the Institute for Genome Sciences and Policy (IGSP), Duke University, Durham, NC 27710. E-mail: \url{r.henao@duke.edu} and \url{joe@stat.duke.edu}. This work was supported by funding from the Defense Advanced Research Projects Agency (DARPA), number lN66001-07-C-0092 (G.S.G.)}
}
\author{Ricardo Henao and Joseph E. Lucas}
\date{\small \today}
\maketitle
%
% \doublespacing
%
\begin{abstract}
	We present an new sequential Monte Carlo sampler for coalescent based Bayesian hierarchical clustering.  Our model is appropriate for modeling non-\iid data and offers a substantial reduction of computational cost when compared to the original sampler without resorting to approximations. We also propose a quadratic complexity approximation that in practice shows almost no loss in performance compared to its counterpart. We show that as a byproduct of our formulation, we obtain a greedy algorithm that exhibits performance improvement over other greedy algorithms, particularly in small data sets. In order to exploit the correlation structure of the data, we describe how to incorporate Gaussian process priors in the model as a flexible way to model non-\iid data. Results on artificial and real data show significant improvements over closely related approaches.
	
\end{abstract}
\section{Introduction}
Learning hierarchical structures from observed data is a common practice in many knowledge domains. Examples include phylogenies and signaling pathways in biology, language models in linguistics, etc.  Agglomerative clustering is still the most popular approach to hierarchical clustering due to its efficiency, ease of implementation and a wide range of possible distance metrics.  However, because it is algorithmic in nature, there is no principled way to that agglomerative clustering can be used as a building block in more complex models. Bayesian priors for structure learning on the other hand, are perfectly suited to be employed in larger models. As an example, several authors have proposed using hierarchical structure priors to model correlation in factor models \citep{rai08,henao12a,zhang11a}.

There are many approaches to hierarchical structure learning already proposed in the literature, see for instance \citet{neal03,heller05,teh08,adams10}. Since we are particularly interested in continuous non \iid data and model based hierarchical clustering we will focus our work on the Bayesian agglomerative clustering model proposed by \citet{teh08}. Although the authors introduce priors both for continuous and discrete data, no attention is paid to the non \iid case, mainly because their work is focused in proposing different inference alternatives.

The approach proposed by \citet{teh08} is a Bayesian hierarchical clustering model with coalescent priors. Kingman's coalescent is a standard model from population genetics perfectly suited for hierarchical clustering since it defines a prior over binary trees \citep{kingman82,kingman82a}. This work advances Bayesian hierarchical clustering in two ways: (i) we extend the original model to handle non \iid data and (ii) we propose an efficient sequential Monte Carlo inference procedure for the model which scales quadratically rather than cubically as in the original sampler by \citet{teh08}. As a byproduct of our approach we propose as well a small correction to the greedy algorithm of \citet{teh08} that shows gains particularly in small data sets.

There is a separate approach by \citet{gorur08} that also improves the cubic computational cost of Bayesian hierarchical clustering. They introduce an efficient sampler with quadratic cost that although proposed for discrete data can be easily extended to continuous data, however as we will show, our approach is still substantially faster.

The remainder of the manuscript is organized as follows, the data model and the use of coalescents as priors for hierarchical clustering are reviewed in Section~\ref{sc:coalescent}. Our approach to inference and relationships to previous approaches are described in Section~\ref{sc:inference}. Section~\ref{sc:results} contains numerical results on both artificial and real data. Section~\ref{sc:discussion} concludes with a discussion and perspectives for future research.
\section{Coalescents for hierarchical clustering} \label{sc:coalescent}
A model for hierarchical clustering consists of learning about a nested set of partitions of $n$ observations in $d$ dimensions, $\X$. Assuming that each partition differs only by two elements, the set of partitions defines a binary tree with $n$ leaves and $n-1$ branching points. Defining $\t=[t_1 \ \ldots \ t_{n-1}]$ and $\bpi=\{\pi_1,\ldots,\pi_{n-1}\}$ as the vector of branching times and the set of partitions, respectively, we can write a Bayesian model for hierarchical clustering as
\begin{align} \label{eq:model}
	\begin{aligned}
	\x_i|\t,\bpi \ \sim & \ p(\x_i|\t,\bpi) \,, \\
	\t,\bpi \ \sim & \ {\rm Coalescent}(n) \,,
	\end{aligned}
\end{align}
where $\x_i$ is the $i$-th row of $\X$, $p(\x_i|\t,\bpi)$ is its likelihood and the pair $\{\t,\bpi\}$ is provided with a prior distribution over binary tree structures known as the coalescent.

\subsection{Kingman's coalescent}
The $n$-coalescent is a continuous-time Markov chain originally introduced to describe the common genealogy of a sample of $n$ individuals backwards in time \citep{kingman82,kingman82a}. It defines a prior over binary trees with $n$ leaves, one for each individual.  The coalescent assumes a uniform prior over tree structures, $\bpi$, and exponential priors on the set of $n-1$ merging times, $\t$.

It can be thought of as a generative process on partitions of $\{1,\ldots,n\}$ as follows
\begin{itemize}
	\item Set $k=1$, $t_0=0$, $\pi_0=\{\{1\},\ldots,\{n\}\}$.
	\item While $k<n$
	\begin{itemize}
		\item Draw $\Delta_k \sim \Exp((n-k+1)(n-k)/2)$ (rate parameter).
		\item Set $t_k=t_{k-1} - \Delta_k$.
		\item Merge uniformly two sets of $\pi_{k-1}$ into $\pi_{k}$.
		\item Set $k=k+1$.
	\end{itemize}
\end{itemize}
Because there are $(n-k+1)(n-k)/2$ possible merges at stage $K$, any particular pair in $\pi_i$ merges with prior rate 1 for any $i$. We can compute the prior probability of a particular configuration of the pair $\{\t,\bpi\}$ as
\begin{align} \label{eq:prior}
	p(\t,\bpi) = \prod_{k=1}^{n-1} \exp\left(-\tfrac{(n-k+1)(n-k)}{2}\Delta_k\right) \,,
\end{align}
this is, the product of merging and coalescing time probabilities. Some properties of the $n$-coalescent include: (i) the marginal distribution of $\bpi$ is uniform and independent of $\t$, (ii) it is exchangeable in the set of partitions $\pi_i$ for every $i$ and (iii) the expected value of $t_{n-1}$ (last coalescing time) is $\mathbb{ E}[t_{n-1}]=2(1-n\inv)$.
\subsubsection{Distribution of the latent nodes}
Let $z_k^*\in \pi_i$ be a node in a binary tree with associated $d$-dimensional vector $\z_k$, and let $z^*_{c_1}$ and $z^*_{c_2}$ be its children.  We designate the leaves of the tree with $x_i^*$.  If we define $p(\z_k|\z_c,\t,\cdot)$ to be the transition density between a child node, $\z_c$, and its parent, $\z_k$, then we can recursively define $q(\z_k|\bpi,\t,\cdot)$ to be a (possibly unnormalized) distribution of $\z_k$ as follows:
\begin{align*}
q(\x_i|\bpi,\t,\cdot) = & \ \delta_{\x_i} \,, \\
q(\z_k|\bpi,\t,\cdot) = & \ \prod_{c\in C}\int p(\z_k|\z_c,\t,\cdot)q(\z_c|\bpi,\t,\cdot)d\z_c \,, \\
= & \ Z_k(\X,\bpi, \t,\cdot) q'(\z_k|\bpi,\t,\cdot) \,,
\end{align*}
where $C=\{c_1,c_2\}$ contains the two sets in $\pi_{k-1}$ that merge in $\pi_k$ and $q'(\z_k|\bpi,\t,\cdot)$ is a density (integrating to 1) and $Z_k$ is the appropriate scaling factor.

Recently, \citet{teh08} showed that by using an agglomerative approach for constructing $\{\t,\bpi\}$, the likelihood for the model in equation~\eqref{eq:model} can be recursively written as
\begin{align} \label{eq:lik}
	p(\X|\t,\bpi,\cdot) = \prod_{k=1}^{n-1} Z_{k}(\X|\bpi,\t,\cdot) \,.
\end{align}
We note that, because of the tree structure, $\z_k$ is independent of $\X$ conditional on the distributions of its two child nodes.  This implies that $Z_k(\X|\bpi,\t,\cdot) = Z_k(\X|\pi_k,\t_{1:k},\cdot)$.  Our formulation is equivalent to using \emph{message passing} to marginalize recursively from the leaves to the root of the tree \citep{pearl88}. The message is $q(\z_k|\cdot)$ for node $\z_k$ and it summarizes the entire subtree below node $\z_k$. 

Figure~\ref{fg:stree} illustrate the process for a segment of a tree. The size of partitions $\pi_k$ shrink as $k$ increases, so from the illustration $\pi_1=\{\{1\},\{2\},\ldots,\{n\}\}$, $\pi_2=\{\{1,2\},\ldots,\{n\}\}$, $\pi_{k-1}=\{c_1,c_2,\ldots\}$, $\pi_k=\{\{c_1,c_2\},\ldots\}$, $\pi_n=\{\{1,2,\ldots,n\}\}$.
\begin{figure}[!t]
	\centering
	\begin{tikzpicture}[ bend angle = 5, >=latex, font = \footnotesize ]
		\tikzstyle{lps} = [ circle, thick, draw = black!80, fill = OliveGreen, minimum size = 1mm, inner sep = 2pt ]
		\tikzstyle{ety} = [ minimum size = 1mm, inner sep = 2pt ]
		\begin{scope}[node distance = 1.5cm and 1cm]
			\node [ety] (t_) [ ] {};
			\node [ety] (t0) [ right of = t_, node distance = 2.5cm ] {$t_1=0$};
			\node [ety] (t1) [ right of = t0 ] {$t_2$};
			\node [ety] (t2) [ right of = t1 ] {$t_{k-1}$};
			\node [ety] (t3) [ right of = t2 ] {};
			\node [ety] (t4) [ right of = t3 ] {$t_k$};
			\node [ety] (t5) [ right of = t4 ] {$t_{k+1}$};
			\node [lps] (lp_1) at (2.5,-0.5) [label = -90:$\x_1$] {};
			\node [lps] (lp_2) at (2.5,-1.5) [label = -90:$\x_2$] {};
			\node [ety] (lp_3) at (2.5,-2.5) [label = -90:$\vdots$] {};
			\node [lps] (lp_4) at (2.5,-4.5) [label = -90:$\x_n$] {};
			\node [lps] (v1) at (4.0,-1.0) [ label = -90:$\z_{c_1}$ ] {}
				edge [pre, bend right] (lp_1) 
				edge [pre, bend left] (lp_2);
			\node[ety] (ve1) at (4.0,-3.0) [ label = 180:$\ldots$ ] {};
			\node[ety] (ve2) at (4.0,-4.0) [ label = 180:$\ldots$ ] {};
			\node [lps] (v2) at (5.5,-3.5) [ label = 90:$\z_{c_2}$ ] {}
				edge [pre, bend right] (ve1)
				edge [pre, bend left] (ve2);
			\node [lps] (v4) at (8.5,-2.5) [ label = -90:$\z_k$ ] {}
				edge [pre, bend right] (v1)
				edge [pre, bend left] (v2);
			\node [ety] (v5) at (10.0,-2.5) [ label = 0:$\ldots$ ] {}
				edge [pre] (v4);
			\node [ety] at (4.9,-0.6) [] {$q(\z_{c_1}|c_1,\t_{1:2})$};
			\node [ety] at (6.6,-3.9) [] {$q(\z_{c_2}|c_2,\t_{1:k-1})$};
			\node [ety] at (9.4,-2.1) [] {$q(\z_{c_1}|C,\t_{1:k})$};
			\node [ety] at (7.5,-1.3) [] {$p(\z_k|\z_{c_1},\t_{1:k})$};
			\node [ety] at (8.2,-3.35) [] {$p(\z_k|\z_{c_2},\t_{1:k})$};
		\end{scope}
	\end{tikzpicture}
	\caption[Segment of the binary tree structure.]{Binary tree structure. Latent variable $\t$ and $\bpi$ define merging points and merging sets, respectively.}
	\label{fg:stree}
\end{figure}
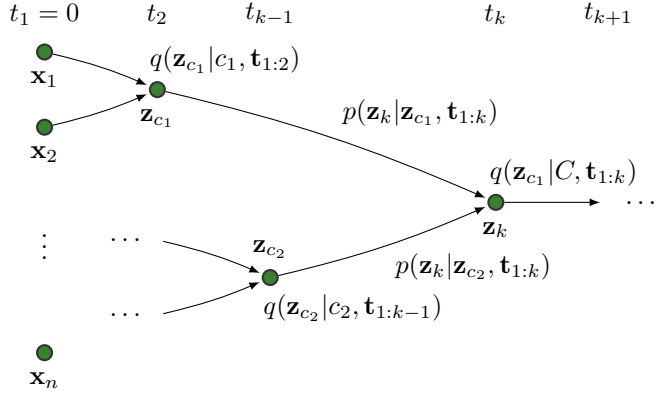

The joint distribution needed to perform inference can be obtained by combining likelihood and prior in equations \eqref{eq:lik} and \eqref{eq:prior} as
\begin{align} \label{eq:joint}
	p(\X,\t,\bpi,\cdot) = \prod_{k=1}^{n-1} \exp\left(-\tfrac{(n-k+1)(n-k)}{2}\Delta_k\right) Z_{k}(\X|\pi_k,\t_{1:k},\cdot) \,.
\end{align}
\subsection{Gaussian transition distributions}
We are interested in correlated continuous data, thus we assume a multivariate Gaussian distribution for the transition probability,
\begin{align*} % \label{eq:trans}
	p(\z_k|\z_c,\t_{1:k},\bPhi) = \DN(\z_k|\z_c,\Delta_c\bPhi) \,,
\end{align*}
where $\bPhi$ is a covariance matrix encoding the correlation structure in $\z_k$ and $\Delta_c$ is the time elapsed between $t_k$ and $t_c$, not necessarily $t_{k-1}-t_{k}$ as can be seen in Figure~\ref{fg:stree}. We denote the time at which the set $c$ was created as $t_c$. Individual terms of the likelihood in equation~\eqref{eq:lik} can be computed using
\begin{align*}
	q(\z_k|\pi_k,\t_{1:k},\cdot) = & \ \DN(\z_k|\m_{c_1},\tilde{s}_{c_1}\bPhi)\DN(\z_k|\m_{c_2},\tilde{s}_{c_2}\bPhi) \,, \\
	= & \ Z_{k}(\X|\pi_k,\t_{1:k},\cdot) \DN(\z_k|s_k(\tilde{s}_{c_1}\inv\m_{c_1}+\tilde{s}_{c_2}\inv\m_{c_2}),s_k\bPhi) \,,
\end{align*}
where $\m_{c_1}$ and $s_{c_1}$ are mean and variance of $q(\z_{c_1}|\cdot)$, respectively, where $\tilde{s}_{c_1}=\Delta_{c_1}+s_{c_1}$, and where $\Delta_{c_1}=t_k-t_{c_1}$.  This leads to $s_k=(\tilde{s}_{c_1}\inv+\tilde{s}_{c_2}\inv)\inv$ and the normalization constant is
\begin{align*}
	Z_{k}(\X|\pi_k,\t_{1:k},\cdot) = & (2\pi)^{-d/2}|v_k\bPhi|^{-1/2}\exp\left(-\tfrac{1}{2}(\m_{c_1}-\m_{c_2})v_k\inv\bPhi\inv(\m_{c_1}-\m_{c_2})\ts\right) \,,
\end{align*}
where $v_k=\tilde{s}_{c_1}+\tilde{s}_{c_2}=2\Delta_k+r_k$ and $r_k=2t_{k-1}-t_{c_1}-t_{c_2}+s_{c_1}+s_{c_2}$. Note that $\Delta_{c_1}=\Delta_{c_2}$ only if $c_1$ and $c_2$ are singletons.
\section{Inference} \label{sc:inference}
Inference is carried out using a sequential Monte Carlo (SMC) sampling based upon equation \eqref{eq:joint} \citep[see][]{doucet01}. More precisely, for a set of $M$ particles, we approximate the posterior of the pair $\{\t,\bpi\}$ using a weighted sum of point masses obtained iteratively by drawing coalescing times $t_k$ and chain states $\pi_k$ one at a time from their posterior as
\begin{align} \label{eq:delta_pi_post}
	p(\Delta_k,\pi_k|\t_{1:k},\pi_{k-1},\cdot) \ = \ Z_{k,C}\inv\exp\left(-\tfrac{(n-k+1)(n-k)}{2}\Delta_k\right) Z_{k}(\X|\pi_k,\t_{1:k},\cdot) \,,
\end{align}
where $Z_{k,C}\inv$ is the normalization constant and $C$ is a pair of elements of $\pi_{k-1}$. In order to compute the weights required by SMC we need to compute $Z_{k,C}\inv$ for every pair in $\pi_{k-1}$.

Algorithms introduced by \citet{teh08} try to avoid the computational complexity of using equation \eqref{eq:delta_pi_post} directly by simplifying it or by means of greedy alternatives. They propose for instance to draw $\Delta_k$ from the prior so computing $Z_{k,C}\inv$ is no longer necessary thus reducing the computational cost. From equation~\eqref{eq:delta_pi_post} we see that $Z_{k,C}\inv$ needs to be computed for every pair in $\pi_{k-1}$ at every iteration of the sampler, simply because the rate of the exponential distribution is a function of $k$. We will show that by using some properties of the distributions involved in equation~\eqref{eq:delta_pi_post} we can effectively decrease the computational complexity of the SMC sampler with almost no performance penalty. In particular, we will show that the most expensive parts of $Z_{k,C}\inv$ need to be computed only once during inference.

We can expand the right hand side of equation \eqref{eq:delta_pi_post} as
\begin{align}
	& p(\Delta_k,\pi_k|\t_{1:k-1},\pi_{k-1},\cdot) \nonumber \\ % \ = & \
	& \hspace{24mm} \propto \ \inv Z_{k}(\X|\pi_k,\t_{1:k},\cdot)\Exp(2\Delta_k+r_k|\lambda/2)\Exp(-r_k|\lambda/2) \,, \nonumber \\
	& \hspace{24mm} = \ Z_{k,C}\inv\GIG(2\Delta_k+r_k|\tilde{\lambda},\bepsilon_{k-1,C},\lambda) \,, \label{eq:delta_pi_modpost}
\end{align}
where $\lambda=(n-k+1)(n-k)/2$, $\tilde{\lambda}=1-d/2$, $\bepsilon_{k-1,C}=(\m_{c_1}-\m_{c_2})\bPhi\inv(\m_{c_1}-\m_{c_2})$, $C=\{c_1,c_2\}\in\pi_{k-1}$, $\GIG(\lambda,\chi,\psi)$ is the generalized inverse Gaussian with parameters $\{\lambda,\chi,\psi\}$ \citep{jorgensen82a}.  This leads to
\begin{align} \label{eq:pi_modpost}
	Z_{k,C} \ \propto & \ \frac{K_{\widetilde{\lambda}}(\sqrt{\lambda\bepsilon_{k-1,C}})}{(\lambda\bepsilon_{k-1,C}\inv)^{\widetilde{\lambda}/2}}\exp\left(\frac{\lambda}{2} r_k\right) \,,
\end{align}
where $K_{\nu}(z)$ is the modified Bessel function of second kind \citep{abramowitz65a}. The details on how to obtain equations \eqref{eq:delta_pi_modpost} and \eqref{eq:pi_modpost} can be found in Appendix~\ref{eq:delta_pi_details}. From Equation \eqref{eq:delta_pi_modpost} we have
\begin{align}
	\Delta_k|\pi_{k},\t_{1:k-1},\cdot \ \sim & \ \GIG(2\Delta_k+r_k|\tilde{\lambda},\bepsilon_{k,C},\lambda) \,, \label{eq:delta_sample} \\
	C^\star|\pi_{k-1},\t_{1:k-1},\cdot \ \sim & \ \Dis(C^\star|\w_{k-1}) \,, \label{eq:pi_sample}
\end{align}
where $\w_{k-1}$ is the vector of normalized weights, ranging over all pairs, computed using equation~\eqref{eq:pi_modpost} and $C^\star$ is the pair of $\pi_{k-1}$ that gets merged in $\pi_k$. Sampling equations \eqref{eq:pi_modpost}, \eqref{eq:delta_sample} and \eqref{eq:pi_sample} have useful properties, (i) the conditional posterior of $\pi_k$ does not depend on $\Delta_k$. (ii) Sampling from $\Delta_k$ amounts to draw from a truncated generalized inverse Gaussian distribution. (iii) We do not need to sample $\Delta_k$ for every pair in $\pi_{k-1}$, in fact we only need to do so for the merging pair $C^\star$. (iv) Although $\lambda$ in equation~\eqref{eq:pi_modpost} changes with $k$, the most expensive computation, $\bepsilon_{k-1,C}$ needs to be computed only once. (v) The distribution in equation~\eqref{eq:pi_modpost} has heavier tails than a Gaussian distribution and $Z_{k,C}\to\infty$ as $\bepsilon_{k-1,C}\to0$ for $d>1$. Furthermore, we can rewrite equation~\eqref{eq:pi_modpost} as
\begin{align}
	Z_{k,C} \ \propto & \ \frac{K_{d/2-1}(\sqrt{\lambda\bepsilon_{k-1,C}})}{(\lambda\inv\bepsilon_{k-1,C})^{-(d-4)/4}}\exp\left(\frac{\lambda}{2} r_k\right) \,, \label{eq:pi_modpost_rep} \\
	Z_{k,C} \ \approx & \ \bepsilon_{k-1,C}^{-(d-1)/4}\exp(-\sqrt{\lambda}\bepsilon_{k-1,C})\exp\left(\frac{\lambda}{2} r_k\right) \label{eq:pi_modpost_ne} \,,
\end{align}
where we have made a change of variables before marginalizing out $2\Delta_k+r_k$ and we have used the limiting form of $K_\nu(z)$ as $z\to\infty$ \citep{abramowitz65a}. \citet{eltoft06a} have called equation~\eqref{eq:pi_modpost_rep} multivariate Laplace distribution. When $d=1$, equation~\eqref{eq:pi_modpost_ne} is exact and is a univariate Laplace distribution. Nevertheless, equation~\eqref{eq:pi_modpost_ne} is particularly useful when $d$ is large as a cheap numerically stable alternative to $K_\nu(z)$.
\subsection{Sampling coalescing times}
Sampling from a generalized inverse Gaussian distribution is traditionally done using the \emph{ratio-of-uniforms} method of \citet{dagpunar89a}. We observed empirically that a slice sampler within the interval $(r_k/2,\Delta_0r_k/2)$ is considerably faster than the commonly used algorithm. Although we use $\Delta_0=10^2$ in all our experiments, we did try larger values without noticing significant changes in the results. The slice sampler used here is a standard implementation of the algorithm described by \citet{neal03a}. We acknowledge that adaptively selecting $\Delta_0$ at each step could improve the efficiency of the sampler however we did not investigate it.
\subsection{Covariance matrix}
Until now we assumed the covariance matrix $\bPhi$ as known, in most cases however the correlation structure of the observed data is hardly available. In practice we need to alternate between SMC sampling for the tree structure and drawing $\bPhi$ from some suitable distribution. For cases when observations exhibit additional structure, such as temporal or spatial data, we may assume the latent variable $\z_k$ is drawn from a Gaussian process with mean $\z_c$ and covariance function $\Delta_kg(i,j,\btheta)$, where entries of $\bPhi$ are computed using $\phi_{ij}=g(i,j,\btheta)$ for a set of hyperparameters $\btheta$. For example, we could use a squared exponential covariance function
\begin{align} \label{eq:sqexp_noise}
	g(i,j,\ell,\sigma^2) = \ \exp\left(-\frac{1}{2\ell}d_{ij}^2\right) + \sigma^2\delta_{ij} \,,
\end{align}
where $\btheta=\{\ell,\sigma^2\}$, $\delta_{ij}=1$ only if $i=j$ and $d_{ij}$ is the time between samples $i$ and $j$. The covariance function in equation~\eqref{eq:sqexp_noise} is a fairly general assumption for continuous signals. The smoothness of the process is controlled by the inverse length scale $\ell$ and the amount of idiosyncratic noise by $\sigma^2$. The elements of $\btheta$ are sampled by coordinate-wise slice sampling using the following function as proxy for the elements of $\btheta$,
\begin{align*}
	f(\btheta|\pi,\t) = \ \sum_{k=1}^{n-1} Z_{k}(\X|\pi_k,\t_{1:k},\cdot) \,.
\end{align*}
For the case when no smoothness is required but correlation structure is expected, conjugate inverse Wishart distributions for $\bPhi$ can be considered. For \iid data, a diagonal/spherical $\bPhi$ with independent inverse gamma priors is a good choice, as already proposed by \citet{teh08}.
\subsection{Greedy implementation}
As pointed out by \citet{teh08}, in some situations, a single good sample from the model is enough. Such a sample can be built by greedily maximizing equation~\eqref{eq:joint} one step at the time. This requires the computation of the mode of $\Delta_k$ from equation~\eqref{eq:delta_sample} for every pair in $\pi_{k-1}$ to then merge the pair with smallest $\Delta_k$, so
\begin{align*}
	\Delta_{k,C} = & \ \frac{1}{2\lambda}\left(\tilde{\lambda}+\sqrt{\tilde{\lambda}^2+\lambda\bepsilon_{k-1,C}}\right)-\frac{1}{2}r_k \,, \\
	C^\star = & \argmin_{\rho}\{\Delta_{k,C},C\in\pi_{k-1}\} \,.
\end{align*}
The greedy algorithm proposed by \citet{teh08} uses instead
\begin{align*}
	\Delta_{k,C}^{\rm prev} = \ \frac{1}{2\lambda}\left(-d+\sqrt{d^2+2\lambda\bepsilon_{k-1,C}}\right)-\frac{1}{2}r_k \,.
\end{align*}
The former proposal uses equation~\eqref{eq:delta_sample} whereas the latter uses equation~\eqref{eq:delta_pi_post} directly without taking into account $Z_{k,C}$. This means that using the properly normalized posterior of $\Delta_k$ leads only to a minor correction of $\Delta_{k,C}$. The following Lemma~\ref{lem:delta_rel} formalizes the relationship between the two proposals.
\begin{lemma}\label{lem:delta_rel}
	If $0<\bepsilon_{k-1,C}<\infty$, then the following holds
	\begin{enumerate}
		\item $\Delta_{k,C}/\Delta_{k,C}^{\rm prev}\to 1$ as $d\to\infty$.
		\item $\Delta_{k,C}/\Delta_{k,C}^{\rm prev}>1$ if $\lambda\bepsilon_{k-1,C}>4(d+2)$.
		\item $\Delta_{k,C}-\Delta_{k,C}^{\rm prev}$ decreases as $2\lambda\bepsilon_{k-1,C}/d^2+{\cal O}(d^{-3})$.
	\end{enumerate}
\end{lemma}
\begin{proof}
	For (1) is enough to take $\lim_{d\to\infty}\Delta_{k,C}/\Delta_{k,C}^{\rm prev}$ by repeatedly applying L'Hospital's rule. (2) is obtained as the solution to the following system of equations
	\begin{align*}
		\widetilde{\Delta}_{k,C}^2+2d\widetilde{\Delta}_{k,C} & = 2\lambda\bepsilon_{k-1,C} \,, \\
		\widetilde{\Delta}_{k,C}^2+(d-2)\widetilde{\Delta}_{k,C} & = \lambda\bepsilon_{k-1,C} \,,
	\end{align*}
	with solution $\Delta_{k,C}=2-r_k/2$ and $\lambda\bepsilon_{k-1,C}=4(d+2)$ for $\widetilde{\Delta}_{k,C}=(2\Delta_{k,C}+r_k)$. Lastly, (3) is the result of a Taylor series expansion of $\Delta_{k,C}-\Delta_{k,C}^{\rm prev}$.
\end{proof}
Lemma~\ref{lem:delta_rel} implies that $\Delta_{k,C}^{\rm prev}$ only matches the true maximum a-posteriori estimate, $\Delta_{k,C}$, when $\lambda\bepsilon_{k-1,C}=4(d+2)$, thus $\Delta_{k,C}^{\rm prev}$ is a biased estimator of $\t$ almost everywhere. Besides, for small values of $d$, the difference between $\Delta_{k,C}$ and $\Delta_{k,C}^{\rm prev}$ could be large enough to make the outcome of both algorithms significantly different.
\subsection{Computational cost}
The computational cost of using directly equation \eqref{eq:delta_pi_post} to sample from $\t$ and $\bpi$ for a single particle is ${\cal O}(\kappa_1n^3)$, where $\kappa_1$ is the cost of drawing the merging time of a single candidate pair, as already mentioned by \citet{teh08}. Using equation \eqref{eq:delta_pi_modpost} costs ${\cal O}(\kappa_2n^3 + \kappa_1n)$, where $\kappa_2$ is the cost of computing $Z_{k,C}$ for a single candidate pair. Since $\kappa_2>>\kappa_1$, using equation \eqref{eq:delta_pi_modpost} is much faster than what proposed by \citet{teh08}, at least for moderately large $n$. From a closer look at equation~\eqref{eq:pi_modpost} we see that the only variables changing with $k$ are $\lambda$ and $r_k$, but also that the only costly operation in it is the modified Bessel function provided we have previously cached $\bepsilon_{:,C}$. We can approximate equation~\eqref{eq:pi_modpost} by
\begin{align} \label{eq:pi_fmodpost}
	Z_{k,C} \ \propto & \ \frac{K_{\widetilde{\lambda}}(\sqrt{\bepsilon_{k-1,C}})}{(\bepsilon_{k-1,C}\inv)^{\widetilde{\lambda}/2}}\exp\left(\frac{\lambda}{2} r_k\right) \,,
\end{align}
this is, we have just dropped $\lambda$ from the Bessel function and the divisor in equation~\eqref{eq:pi_modpost}, which is acceptable because (i) $K_{\nu}(z)$ is strictly decreasing for fixed $\nu$ and (ii) the $\lambda$ term appearing in the divisor is a constant in $\log Z_{k,C}$. Note that a similar reasoning can be applied to equation~\eqref{eq:pi_modpost_ne}, which is cheaper and more numerically stable. Since equation~\eqref{eq:pi_fmodpost} does depend on $k$ only through $\lambda r_k$, we have virtually decreased the cost from ${\cal O}(\kappa_2n^3 +\kappa_3n)$ to ${\cal O}(\kappa_2n^2 + \kappa_3n)$, that is, we need to compute equation~\eqref{eq:pi_fmodpost} for every possible pair only once before selecting the merging pair at stage $k$, then we add $\lambda r_k/2$ (in log-domain) before sampling its merging time. From now on we use {\sc MPost1} to refer to the algorithm using equation \eqref{eq:delta_pi_modpost} and {\sc MPost2} to the fast approximation in equation~\eqref{eq:pi_fmodpost}.

Recently, \citet{gorur08} proposed an efficient SMC sampler ({\sc SMC1}) for hierarchical clustering with coalescents that although was introduced for discrete data can be easily adapted to continuous data. Their approach is based in a regenerative race process in which every possible pair candidate proposes a merging time only once leading to ${\cal O}(\kappa_1n^2)$ computational time. In principle, {\sc SMC1} has quadratic cost as {\sc MPost2}, however since $\kappa_2>>\kappa_1$ our approximation is considerably faster. In addition, we have observed empirically that at least for $n$ in the lower hundreds, {\sc MPost1} is faster than {\sc SMC1} as well.

The key difference between our approach and that of \citet{gorur08} is that the latter proposes merging times for every possible pair and selects the pair to merge as the minimum available at a given stage whereas our approach selects the pair to merge and samples the merging time independently. Additionally, they do not sample merging times using $\Delta_k=t_k-t_{k-1}$ but directly $t_k|t_C$, where $t_C$ is the time at which the pair $C$ was created, thus $0\leq t_C < t_k$. As $t_C$ is usually smaller than $t_{k-1}$, {\sc SMC1} draw the vector $\t$ in larger jumps compared to {\sc MPost1/2}. This suggests that our approach will have in general better mixing properties as we will show empirically in the next section.
\section{Numerical results} \label{sc:results}
In this section we consider a number of experiments on both artificial and real data to highlight the benefits of the proposed approaches as well as to compare them with previously proposed ones. In total three artificial data based simulations and two real data set applications are presented. All experiments are obtained using a desktop machine with 2.8GHz processor with 8GB RAM and run times are measured as single core CPU times.
\subsection{Artificial data - structure}
First we compare different sampling algorithms on artificially generated data using $n$-coalescents and Gaussian processes with known squared exponential covariance functions as priors. We generated 50 replicates of two different settings, $D_1$ and $D_2$ of sizes $\{n,d\}=\{32,32\}$ and $\{64,64\}$, respectively. We compare four different algorithms, {\sc Post-Post} \citep{teh08}, {\sc SCM1} \citep{gorur08}, {\sc MPost1} and {\sc MPost2}. In each case we collect $M=100$ particles and set the covariance function parameters $\{\ell,\sigma^2\}$ to their true values. As performance measures we track runtime ({\sc rt}) as proxy to the computational cost, mean squared error ({\sc mse}), mean absolute error ({\sc mae}) and maximum absolute bias ({\sc mab}) of $\t$ and $\bpi$ in log domain, as $\t$ grows exponentially fast. For the latter we compute the distance matrix encoded by $\{\t,\bpi\}$. Table~\ref{tb:artificial_structure} shows performance measures averaged over 50 replicates for each data set. In terms of error, we see that all four algorithms perform about the same, however with {\sc MPost1} and {\sc SMC1} being slightly better and slightly worse, respectively. The computational cost is significantly higher for the {\sc Post-Post} approach whereas {\sc MPost2} is the fastest. We see {\sc MPost1} and {\sc MPost2} consistently outperforming the other two algorithms as an indication of better mixing properties. In more general terms, {\sc MPost2} provides the best error/computational cost trade-off as the difference in accuracy between {\sc MPost1} and {\sc MPost2} is rather minimal.
\begin{table}[!t]
	\centering
	\begin{tabular}{cccccc}
	\hline
	Set & Measure & {\sc Post-post} & {\sc MPost1} & {\sc SMC1} & {\sc MPost2} \\
	\hline
	\multicolumn{6}{l}{Merge time ($\t$)} \\
	\multirow{3}{*}{$D_1$} & $10^{1}\times${\sc mse} & $0.52\pm0.22$ & $\bf 0.44\pm0.18$ & $0.66\pm0.28$ & $0.45\pm0.18$ \\
	& $10^{1}\times${\sc mae} & $1.88\pm0.45$ & $\bf 1.68\pm0.39$ & $2.01\pm0.46$ & $1.72\pm0.40$ \\
	& $10^{1}\times${\sc mab} & $5.13\pm1.23$ & $4.93\pm1.12$ & $6.09\pm1.30$ & $\bf 4.91\pm1.08$ \\
	\multirow{3}{*}{$D_2$} & $10^{2}\times${\sc mse} & $4.06\pm1.08$ & $\bf 3.04\pm0.90$ & $5.37\pm2.23$ & $3.30\pm0.94$ \\
	& $10^{1}\times${\sc mae} & $1.73\pm0.26$ & $\bf 1.42\pm0.26$ & $1.83\pm0.39$ & $1.49\pm0.27$ \\
	& $10^{1}\times${\sc mab} & $4.47\pm0.73$ & $4.40\pm0.78$ & $5.97\pm1.39$ & $\bf 4.39\pm0.73$ \\
	\multicolumn{6}{l}{Distance matrix ($\bpi$)} \\
	\multirow{3}{*}{$D_1$} & $10^{1}\times${\sc mse} & $0.79\pm0.50$ & $0.70\pm0.49$ & $1.32\pm0.63$ & $\bf 0.70\pm0.42$ \\
	& $10^{1}\times${\sc mae} & $2.24\pm0.78$ & $\bf 2.13\pm0.76$ & $3.03\pm0.90$ & $2.14\pm0.72$ \\
	& $10^{1}\times${\sc mab} & $6.77\pm1.20$ & $6.50\pm1.12$ & $8.77\pm1.96$ & $\bf 6.48\pm1.13$ \\
	\multirow{3}{*}{$D_2$} & $10^{2}\times${\sc mse} & $5.79\pm3.15$ & $\bf 4.89\pm2.92$ & $11.36\pm5.68$ & $5.27\pm2.94$ \\
	& $10^{1}\times${\sc mae} & $1.95\pm0.68$ & $\bf 1.78\pm0.65$ & $2.85\pm0.83$ & $1.85\pm0.65$ \\
	& $10^{1}\times${\sc mab} & $6.06\pm0.79$ & $\bf 5.75\pm0.73$ & $8.54\pm2.01$ & $5.81\pm0.69$ \\
	\multicolumn{6}{l}{Computational cost} \\
	\multirow{1}{*}{$D_1$} & $10^{0}\times${\sc rt} & $18.65\pm0.24$ & $2.29\pm0.04$ & $3.76\pm0.07$ & $\bf 1.98\pm0.03$ \\
	\multirow{1}{*}{$D_2$} & $10^{-1}\times${\sc rt} & $14.50\pm0.06$ & $1.08\pm0.00$ & $1.39\pm0.01$ & $\bf 0.61\pm0.00$ \\
	\hline
	\end{tabular}
	\caption{Performance measures for structure estimation. {\sc mse}, {\sc mae}, {\sc mab} and {\sc rt} are mean squared error, mean absolute error and maximum absolute bias, and runtime in seconds, respectively. Figures are means and standard deviations across 50 replicates. Best results are in boldface letters.}
	\label{tb:artificial_structure}
\end{table}

\subsection{Artificial data - covariance}
Next we want to test the different sampling algorithms when the parameters of the Gaussian process covariance function, $\{\ell,\sigma^2\}$, need to be learned as well. We use settings similar to those in the previous experiment with the difference that now $M=50$ particles are collected and $N_{\rm iter}=50$ iterations are performed to learn the covariance matrix parameters. We dropped the first 10 iterations as burn-in period. In addition to the previously mentioned performance measures we also compute {\sc mse}, {\sc mae} and {\sc mab} for the inverse length scale $\ell$ from equation~\eqref{eq:sqexp_noise}. In order to simplify the experiment we set a priori $\sigma^2=1\times10^{-9}$ to match a \emph{noiseless} scenario, however similar results are obtained when learning both parameters at the same time (results not shown). Table~\ref{tb:artificial_covariance} shows an overall similar trend when compared to Table~\ref{tb:artificial_structure}. In terms of covariance function parameter estimation, we see all algorithms perform about the same which is not surprising considering they use the same sampling strategy.
\begin{table}[!t]
	\centering
	\begin{tabular}{ccccc}
	\hline
	Set & Measure & {\sc MPost1} & {\sc SMC1} & {\sc MPost2} \\
	\hline
	\multicolumn{5}{l}{Merge time ($\t$)} \\
	\multirow{3}{*}{$D_1$} & $10^{1}\times${\sc mse} & $\bf 1.20\pm0.51$ & $1.26\pm0.44$ & $1.24\pm0.53$ \\
	& $10^{1}\times${\sc mae} & $3.02\pm0.80$ & $\bf 2.98\pm0.62$ & $3.06\pm0.81$ \\
	& $10^{1}\times${\sc mab} & $\bf 6.43\pm1.26$ & $7.01\pm1.50$ & $6.52\pm1.38$ \\
	\multirow{3}{*}{$D_2$} & $10^{2}\times${\sc mse} & $\bf 5.38\pm1.66$ & $6.32\pm1.94$ & $5.61\pm1.76$ \\
	& $10^{1}\times${\sc mae} & $2.02\pm0.36$ & $\bf 2.01\pm0.36$ & $2.07\pm0.36$ \\
	& $10^{1}\times${\sc mab} & $4.75\pm0.72$ & $6.16\pm1.35$ & $\bf 4.73\pm0.62$ \\
	\multicolumn{5}{l}{Distance matrix ($\bpi$)} \\
	\multirow{3}{*}{$D_1$} & $10^{1}\times${\sc mse} & $\bf 1.31\pm0.87$ & $2.85\pm1.76$ & $1.33\pm0.86$ \\
	& $10^{1}\times${\sc mae} & $\bf 2.94\pm1.19$ & $4.53\pm1.69$ & $2.96\pm1.19$ \\
	& $10^{1}\times${\sc mab} & $\bf 8.77\pm2.18$ & $9.60\pm1.73$ & $8.84\pm2.13$ \\
	\multirow{3}{*}{$D_2$} & $10^{1}\times${\sc mse} & $\bf 0.64\pm0.35$ & $1.54\pm0.67$ & $0.66\pm0.34$ \\
	& $10^{1}\times${\sc mae} & $\bf 2.08\pm0.63$ & $3.29\pm0.94$ & $2.08\pm0.65$ \\
	& $10^{1}\times${\sc mab} & $6.77\pm1.13$ & $8.41\pm1.42$ & $\bf 6.76\pm1.15$ \\
	\multicolumn{5}{l}{Inverse length scale ($\ell$)} \\
	\multirow{3}{*}{$D_1$} & $10^{4}\times${\sc mse} & $\bf 2.36\pm3.20$ & $2.93\pm4.51$ & $2.38\pm3.23$ \\
	& $10^{2}\times${\sc mae} & $\bf 1.17\pm0.99$ & $1.24\pm1.09$ & $1.18\pm0.99$ \\
	& $10^{2}\times${\sc mab} & $1.40\pm1.14$ & $2.06\pm2.18$ & $\bf 1.38\pm1.15$ \\
	\multirow{3}{*}{$D_2$} & $10^{4}\times${\sc mse} & $2.86\pm3.22$ & $3.55\pm4.48$ & $\bf 2.83\pm3.17$ \\
	& $10^{2}\times${\sc mae} & $1.35\pm1.02$ & $1.44\pm1.14$ & $\bf 1.34\pm1.01$ \\
	& $10^{2}\times${\sc mab} & $1.57\pm1.29$ & $2.39\pm2.24$ & $\bf 1.55\pm1.27$ \\
	\multicolumn{3}{l}{Computational cost} & \\
	\multirow{1}{*}{$D_1$} & $10^{-1}\times${\sc rt} & $5.69\pm0.02$ & $13.65\pm0.06$ & $\bf 4.86\pm0.02$ \\
	\multirow{1}{*}{$D_2$} & $10^{-2}\times${\sc rt} & $2.72\pm0.07$ & $5.12\pm0.05$ & $\bf 1.49\pm0.01$ \\
	\hline
	\end{tabular}
	\vspace{1mm}
	\caption{Performance measures for covariance estimation. {\sc mse}, {\sc mae}, {\sc mab} and {\sc rt} are mean squared error, mean absolute error and maximum absolute bias, and runtime in seconds, respectively. Figures are means and standard deviations across 50 replicates. Best results are in boldface letters.}
	\label{tb:artificial_covariance}
\end{table}

\subsection{Artificial data - greedy algorithm}
As final simulation based on artificial data we want to test wether the correction to the greedy algorithm of \citet{teh08} makes any differences performance-wise. We generated 50 replicates of two different settings  $D_1$ and $D_2$ of sizes $\{n,d\}=\{32,32\}$ and $\{128,128\}$, respectively. We run $N_{\rm iter}$ iterations of the algorithm and drop the first 10 samples as burn-in period. The performance measures are the same as in the previous experiment. Table~\ref{tb:artificial_greedy} shows that the corrected algorithm performs consistently better than the original when the data set is small. When the data set is larger the difference between the two algorithms diminishes however still favors {\sc MGreedy}. Although not shown, we tried other settings in between $D_1$, $D_2$ and larger than $D_2$ with consistent results, this is, the difference between {\sc Greedy} and {\sc MGreedy} decreases with the size of the dataset.
\begin{table}[!t]
	\centering
	\begin{tabular}{cccccc}
		\hline
		\multicolumn{3}{c}{$D_1$} & \multicolumn{3}{c}{$D_2$} \\
		Measure & {\sc Greedy} & {\sc MGreedy} & Measure & {\sc Greedy} & {\sc MGreedy} \\
		\hline
		\multicolumn{6}{c}{Merge time ($\t$)} \\
		$10^{1}\times${\sc mse} & $1.26\pm0.48$ & $\bf 0.90\pm0.36$ & $10^{2}\times${\sc mse} & $0.98\pm0.23$ & $\bf 0.78\pm0.19$ \\
		$10^{1}\times${\sc mae} & $3.10\pm0.67$ & $\bf 2.53\pm0.59$ & $10^{1}\times${\sc mae} & $0.81\pm0.11$ & $\bf 0.71\pm0.10$ \\
		$10^{1}\times${\sc mab} & $7.03\pm1.54$ & $\bf 6.38\pm1.50$ & $10^{1}\times${\sc mab} & $2.74\pm0.49$ & $\bf 2.59\pm0.49$ \\
		\multicolumn{6}{c}{Distance matrix ($\bpi$)} \\
		$10^{1}\times${\sc mse} & $1.33\pm0.74$ & $\bf 1.02\pm0.60$ & $10^{1}\times${\sc mse} & $0.14\pm0.11$ & $\bf 0.12\pm0.09$ \\
		$10^{1}\times${\sc mae} & $3.04\pm1.02$ & $\bf 2.58\pm0.93$ & $10^{1}\times${\sc mae} & $0.95\pm0.38$ & $\bf 0.87\pm0.35$ \\
		$10^{1}\times${\sc mab} & $8.90\pm1.70$ & $\bf 8.24\pm1.69$ & $10^{1}\times${\sc mab} & $4.05\pm0.69$ & $\bf 3.90\pm0.68$ \\
		\multicolumn{6}{c}{Inverse length scale ($\ell$)} \\
		$10^{4}\times${\sc mse} & $2.03\pm2.61$ & $\bf 2.03\pm2.58$ & $10^{4}\times${\sc mse} & $\bf 2.79\pm3.54$ & $2.79\pm3.56$ \\
		$10^{2}\times${\sc mae} & $1.10\pm0.91$ & $\bf 1.10\pm0.90$ & $10^{2}\times${\sc mae} & $\bf 1.32\pm1.03$ & $\bf 1.32\pm1.03$ \\
		$10^{2}\times${\sc mab} & $1.24\pm1.00$ & $\bf 1.24\pm0.99$ & $10^{2}\times${\sc mab} & $\bf 1.39\pm1.06$ & $1.43\pm1.15$ \\
		\multicolumn{6}{c}{Computational cost} \\
		$10^{0}\times${\sc rt} & $\bf 1.65\pm0.09$ & $1.66\pm0.08$ & $10^{-1}\times${\sc rt} & $3.86\pm2.42$ & $\bf 3.79\pm2.31$ \\
		\hline
	\end{tabular}
	\vspace{1mm}
	\caption{Performance measures for greedy algorithms. {\sc mse}, {\sc mae}, {\sc mab} and {\sc rt} are mean squared error, mean absolute error and maximum absolute bias, and runtime in seconds, respectively. Figures are means and standard deviations across 50 replicates. Best results are in boldface letters.}
	\label{tb:artificial_greedy}
\end{table}

\subsection{Handwritten digits}
The USPS database\footnote{Data available from \url{http://cs.nyu.edu/~roweis/data.html}.} contains 9289 grayscale images of $16\times16$ pixels in size, scaled to fall within the range $[-1,1]$. Here we use subsets of 500 images, 50 from each digit, randomly selected from the full data set. We apply {\sc MPost2}, {\sc MGreedy} and average-link agglomerative clustering ({\sc HC}) to 25 of such subsets. For the covariance matrix $\bPhi$ we use a Matérn covariance function with parameter $\nu=3/2$ and additive noise defined as follows
\begin{align*}
	g(i,j,\ell_x,\ell_y,\sigma^2) = \left(1+\frac{\sqrt{3}}{\ell_x}d_{x,ij}\right)\left(1+\frac{\sqrt{3}}{\ell_y}d_{y,ij}\right)\exp\left( -\frac{\sqrt{3}}{\ell_x}d_{x,ij}+\frac{\sqrt{3}}{\ell_y}d_{y,ij} \right) + \sigma^2\delta_{ij} \,,
\end{align*}
where $d_{x,ij}$ and $d_{y,ij}$ are distances in the two axes of the image, and we have assumed axis-wise independency \citep{rasmussen06}.
\subsubsection{Performance metrics}
As performance measures we use (i) the \emph{subtree} score defined as $N_{\rm subset}/(n-C)$, where $N_{\rm subset}$ is the number of internal nodes with leaves from the same class and $C$ is the number of classes \citep{teh08}, and (ii) the area under the adjusted Rand index (ARI) curve (AUC). ARI is a similarity measure for pairs of data partitions that take values between 0 and 1, the latter indicating that the two partitions are exactly the same \citep{hubert85a}. Although ARI has been extensively used for clustering assessment, its use in hierarchical clustering requires the tree to be cut to obtain a single partition of data. Provided we can obtain $n$ different partitions from hierarchical clustering on $n$ observations, we can compute ARI for all partitions using a majority voting rule to label internal nodes of the tree structure. If we plot ARI vs number of clusters $N_c$ we obtain a graphical representation that resembles a ROC curve. When all partitions have a correct label, ARI will be 1 for $N_c>C$ and in any other case, ARI will increase with $N_c$. Besides, when $N_c=1$ and $N_c=n$, ARI is always 0 and 1, respectively. Just like in a ROC curve, an algorithm is as good as its ARI's rate of change thus we can asses the overall performance by computing the area under the ARI curve. Figure~\ref{fg:usps_auc} shows curves for a particular data set and the three considered algorithms. We also included results obtained by tree structures drawn from the coalescent prior as reference.
\begin{figure}[!t]
	\begin{minipage}[c]{0.38\linewidth}
		\centering
		\begin{psfrags}
			\psfrag{ari}[c][c][0.5][0]{ARI}
			\psfrag{nc}[c][c][0.5][0]{$N_c$}
			\psfrag{prior}[c][c][0.4][0]{\hspace{19mm}\sc Prior (0.36)}
			\psfrag{hc}[c][c][0.4][0]{\hspace{19mm}\sc HC (0.83)}
			\psfrag{mgreedy (x.xxxx)}[c][c][0.4][0]{\sc MGreedy (0.86)}
			\psfrag{mpost2}[c][c][0.4][0]{\hspace{18mm}\sc MPost2 (0.88)}
			\includegraphics[scale=0.35]{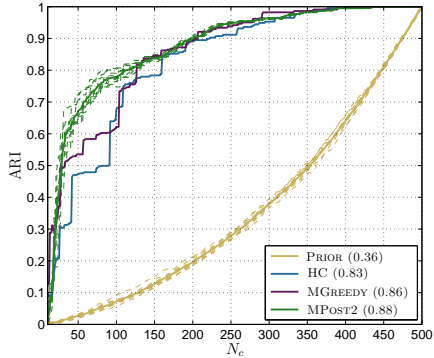}
		\end{psfrags}
	\end{minipage}
	\begin{minipage}[c]{0.6\linewidth}
		\centering
		\begin{tabular}{cccc}
			\hline
			& {\sc MPost2} & {\sc MGreedy} & {\sc HC} \\
			\hline
			Subtree & $0.77\pm0.02$ & $\bf 0.78\pm0.02$ & $0.76\pm0.02$ \\
			AUC & $\bf 0.89\pm0.01$ & $0.88\pm0.02$ & $0.85\pm0.02$ \\ 
			RT & $17.30\pm0.17$ & $2.65\pm0.05$ & $<1.00$ \\
			\hline
		\end{tabular}
	\end{minipage}
	\caption{USPS digits results. (Left) USPS data ARI curves. {\sc Prior} draws structures directly from a $n$-coalescent prior and {\sc HC} is standard hierarchical clustering with average link function and euclidean distance metric. Figures in parenthesis are AUC scores. (Right) Subtree scores, AUC is area under the ARI curve and RT is runtime in minutes. Figures are means and standard deviations across 25 replicates. Best results are in boldface letters.}
	\label{fg:usps_auc}
\end{figure}

Table in Figure~\ref{fg:usps_auc} shows average subset scores for the algorithms considered. We performed inference for 20 iterations and 10 particles. Not substantial improvement was found by increasing $N_{\rm iter}$ or $M$. {\sc Greedy} produced the same results as {\sc MGreedy} and {\sc SMC1} did not performed better than {\sc MPoost1/2} but it took approximately 10 times longer to run (results not shown). We see that coalescent based algorithms perform considerable better than standard hierarchical cluster in terms of AUC. Besides, {\sc MPost1} and {\sc MPost2} are best in subtree scores and AUC, respectively.
\subsection{Motion capture data}
We apply {\sc MPost2} to learn hierarchical structures in motion capture data (MOCAP). The data set consist of 102 time series of length 217 corresponding to the coordinates of a set of 34 three dimensional markers placed on a person breaking into run\footnote{Data available from \url{http://accad.osu.edu/research/mocap/mocap_data.htm}.}. For the covariance matrix $\bPhi$, we used the squared exponential function in equation~\eqref{eq:sqexp_noise}. Results are obtained after running 50 iterations of {\sc MPost2} with 50 particles. It took approximately 5 minutes to complete the run. Left panel in Figure~\ref{fg:mocap_tree} shows two subtrees containing data from all markers in the $X$ and $Z$ axes. Aiming to facilitate visualization, we relabeled the original markers to one of the following: head (Head), torso (Torso), right leg (Leg:R), left leg (Leg:L), right arm (Arm:R) and left arm (Arm:L). The subtrees from Figure~\ref{fg:mocap_tree} are obtained from the particle with maximum weight (0.129) at the final iteration of the run with effective sample size 24.108. We also examined the trees for the remaining particles and noted no substantial structural differences with respect to Figure~\ref{fg:mocap_tree}. The resulting tree has interesting features: (i) Sensors from different coordinates are put together. (ii) Leg markers have in general larger merging times than the others, whereas the opposite is true for head markers. (iii) The obtained tree fairly agrees with the structure of the human body, for instance in the middle-right panel of Figure~\ref{fg:mocap_tree} we see a heat map with 9 markers, 4 of them from the head, 1 from the torso (C7, base of the neck) and 4 from the arms (shoulders and upper arms). The two arm sensors close to the torso correspond to the shoulders while the other two---with larger merging times, are located in the upper arms. We observed that the obtained structure is fairly robust to changes in the number of iterations and particles. We also point out that {\sc MPost1}, {\sc GreedyNew}, {\sc SMC1}, and {\sc Post} produce structurally similar trees to the one shown in Figure~\ref{fg:mocap_tree}, however with different running times. In particular, they took 7, 1, 12 and 75 minutes, respectively.
\begin{figure}[!t]
	\centering
		\begin{psfrags}
			\psfrag{leg:l:z}[c][c][0.4][0]{Leg:L:Z}
			\psfrag{leg:l:y}[c][c][0.4][0]{Leg:L:Y}
			\psfrag{leg:r:z}[c][c][0.4][0]{Leg:R:Z}
			\psfrag{leg:r:y}[c][c][0.4][0]{Leg:R:Y}
			\psfrag{arm:l:z}[c][c][0.4][0]{Arm:L:Z}
			\psfrag{arm:l:y}[c][c][0.4][0]{Arm:L:Y}
			\psfrag{arm:r:z}[c][c][0.4][0]{Arm:R:Z}
			\psfrag{arm:r:y}[c][c][0.4][0]{Arm:R:Y}
			\psfrag{head:z}[c][c][0.4][0]{Head:Z}
			\psfrag{head:y}[c][c][0.4][0]{Head:Y}
			\psfrag{torso:z}[c][c][0.4][0]{Torso:Z}
			\psfrag{torso:y}[c][c][0.4][0]{Torso:Y}
			\includegraphics[scale=0.35]{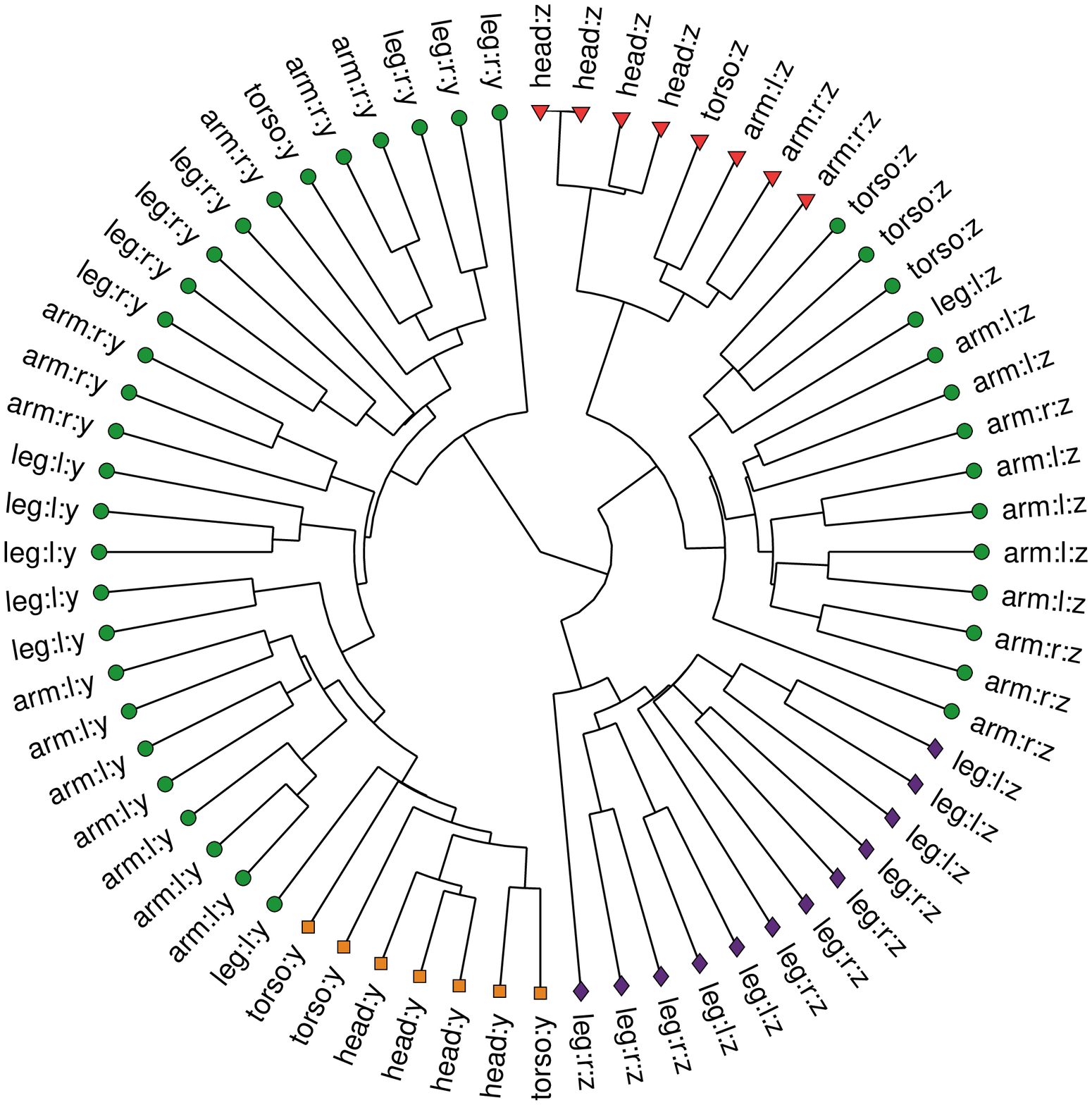} \hspace{2mm}
			\includegraphics[scale=0.40]{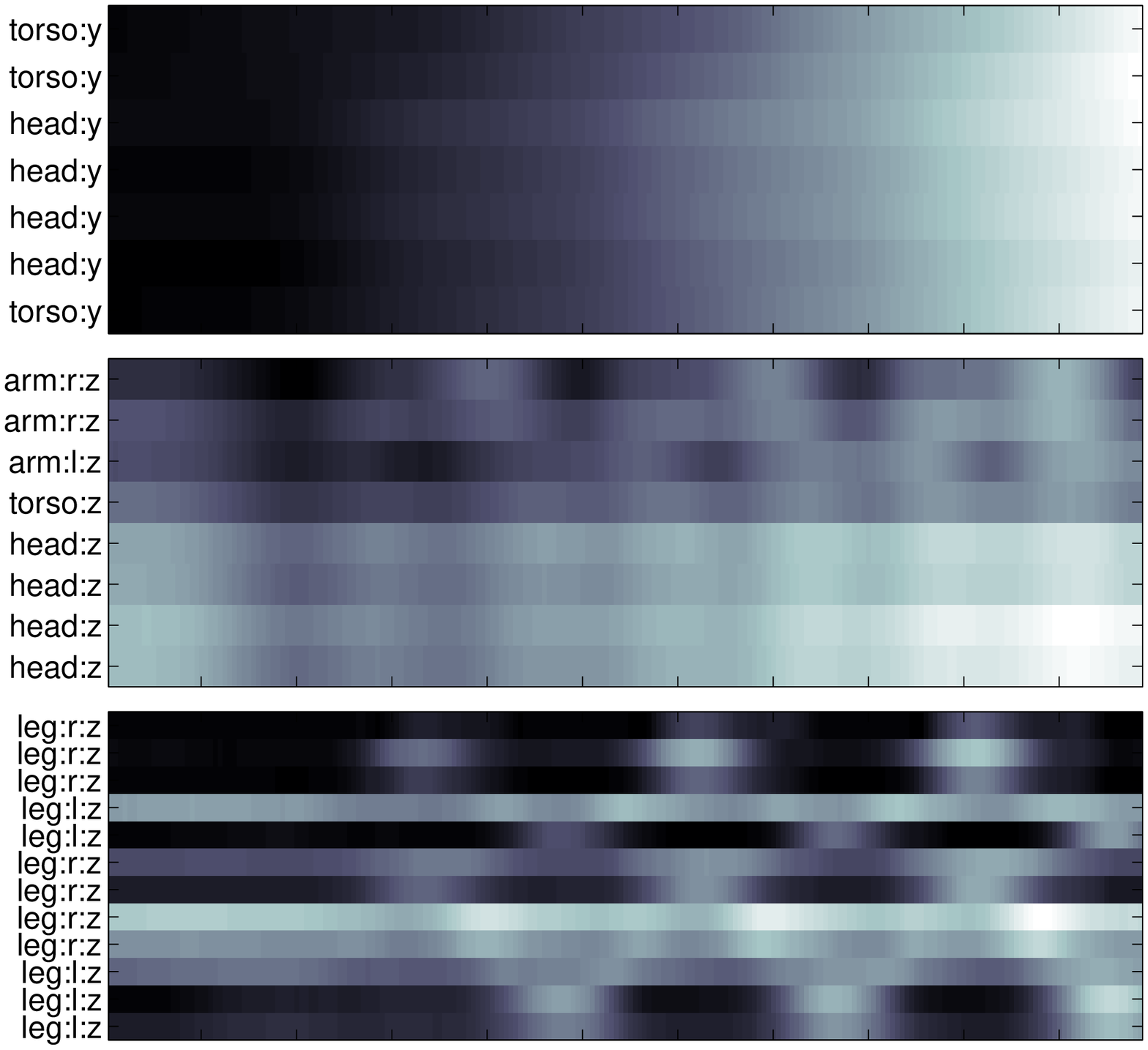}
		\end{psfrags}
	\caption{MOCAP data results. (Left) Resulting subtree from the particle with maximum weight (0.0784) at iteration 50. (Right) Data corresponding to three subtrees of (Left) marked with squares, triangles and diamonds, respectively.}
	\label{fg:mocap_tree}
\end{figure}
\section{Discussion} \label{sc:discussion}
The model for hierarchical clustering for continuous data presented in this paper has shown to perform better than the other alternatives available in the literature with appealing reduced computational cost. We also showed that the approximation {\sc MPost2} has nearly the same performance as the exact version {\sc MPost1} but with much improved computational cost and that the greedy strategy obtained as the mode of the merging time posterior behaves better than the algorithm proposed by \citet{teh08}. Experiments with real data highlight the flexibility of our model by means of exploiting covariation through the use of Gaussian process priors.

Some opportunities for future research being considered include: use the hierarchical cluster model as a prior for the factors in a factor model for better handling of time series or replicates in biological data. The model itself could see improvements if allowed for arbitrary branching structures and stage specific length scales.
\section*{Supplementary materials}
{\bf Source code:} Matlab/C code required to perform the methods described in the manuscript, including sample scripts and the artificial data generator used throughout the results section (\url{mvbtree_v0.1.zip}, compressed file).
%
% \section*{Acknowledgment}
%
%\iffalse
\appendix
\section{Tree posteriors details} \label{eq:delta_pi_details}
We can write the first line of equation~\eqref{eq:delta_pi_modpost} as
\begin{align}
	& = \ p(\Delta_k|\pi_k,\t_{1:k-1},\cdot)p(\pi_k|\t_{1:k-1},\pi_{k-1},\cdot) \,, \label{eq:pdelta_ppi} \\
	& = \ \DN\left(\m\middle|\0,v_k\bPhi\right)\Exp\left(v_k\middle|\lambda/2\right)\Exp(-r_k|\lambda/2) \,, \nonumber \\
	& = \frac{\lambda}{2}(2\pi)^{-d/2}|\bPhi|^{-1/2}\underbrace{|v_k|^{-1/2}\exp\left(-\frac{1}{2}v_k\inv\bepsilon_{k-1,C}-\frac{1}{2}\lambda v_k\right)}\frac{\lambda}{2}\exp\left(\frac{\lambda}{2}r_k\right) \,, \label{eq:gig_core}
\end{align}
where $\lambda=(n-k+1)(n-k)/2$, $\m=\m_{c_1}-\m_{c_2}$ and $\bepsilon_{k-1,C}=\m\bPhi\inv\m\ts$. We recognize the segment in braces of equation~\eqref{eq:gig_core} as the core of a generalized inverse Gaussian distribution \citep{jorgensen82a} defined as
\begin{align*}
	{\rm GIG}(x|\lambda,\chi,\psi) \ = \ 	\frac{(\psi\chi\inv)^{\lambda/2}}{2K_\lambda(\sqrt{\psi\chi})}x^{\lambda-1}\exp\left(-\frac{\chi}{2}x\inv-\frac{\psi}{2}x\right) \,,
\end{align*}
where $K_\nu(z)$ is the modified Bessel function of second kind with parameter $\nu$.

We can thus rewrite equation~\eqref{eq:pdelta_ppi} as
\begin{align*}
	\Delta_k|\pi_k,\t_{1:k-1},\cdot \ \sim & \ {\rm GIG}(v_k|\widetilde{\lambda},\bepsilon_{k,C},\lambda) \,, \\
	\pi_k|\t_{1:k-1},\pi_{k-1},\cdot  \ \propto & \ \underbrace{\frac{\lambda^2}{4}(2\pi)^{-d/2}\frac{2K_{\widetilde{\lambda}}(\sqrt{\lambda\bepsilon_{k-1,C}})}{(\lambda\bepsilon_{k-1,C}\inv)^{\widetilde{\lambda}/2}}|\bPhi|^{-1/2}\exp(\frac{\lambda}{2}r_k)}_{Z_{k,C}} \,,
\end{align*}
where $\widetilde{\lambda}=1-d/2$. % and , $\mathbb{I}(v_k>r_k)$ is the indicator function and $\Delta_k=(v_k-r_k)/2$.
%
\iffalse
Marginalizing out $v_k$ we have, $\mathbb{I}(v_k>r_k)$ 
%
\begin{align*}
	Z_{k}(\X|\pi_k,\t_{1:k},\cdot) = & \ (2\pi)^{-d/2}|v_k|^{-1/2}|\bPhi|^{-1/2}\exp\left(-\tfrac{1}{2}v_k\inv\bepsilon_{k-1,\rho}\right) \,, \\
	\Delta_k|\lambda \ = & \ \frac{\lambda}{2}\exp\left(-\frac{\lambda}{2}\Delta_k\right) \,,
\end{align*}
%
\begin{align*}
	Z_{k,\rho} \ = & \ \int Z_{k}(\X|\pi_k,\t_{1:k},\cdot)p(\Delta_k|\lambda)dv_k \,, \\
	& \int \frac{\lambda}{2}(2\pi)^{-N/2}v_k^{-N/2}|\bPhi|^{-1/2}\exp\left(-\frac{\lambda}{2}v_k-\frac{\bepsilon_{k-1,\rho}}{2}v_k\inv\right)\exp\left(-\frac{\lambda}{2}r_k\right)dv_k \,.
\end{align*}
%
, we obtain
%
\begin{align*}
	v_k|\bepsilon_{k-1,\rho},\bPhi,\lambda,v_k>r_k \ \sim & \ {\rm GIG}(v_k|\widetilde{\lambda},\bepsilon,\lambda/2)\mathbb{I}(v_k>r_k) \,, \\
	Z_{k,\rho} \ = & \ \frac{\lambda}{2}(2\pi)^{-N/2}\frac{2K_{\widetilde{\lambda}}(\sqrt{\lambda\bepsilon/2})}{(\lambda\bepsilon\inv/2)^{\widetilde{\lambda}/2}}|\bPhi|^{-1/2}\exp(\frac{\lambda}{4}r_k) \,,
\end{align*}
%
\fi
%
% \bibliography{../bib_files/mlbib}
\bibliography{./mlbib}

\begin{thebibliography}{19}
\providecommand{\natexlab}[1]{#1}
\providecommand{\url}[1]{\texttt{#1}}
\expandafter\ifx\csname urlstyle\endcsname\relax
  \providecommand{\doi}[1]{doi: #1}\else
  \providecommand{\doi}{doi: \begingroup \urlstyle{rm}\Url}\fi

\bibitem[Abramowitz and Stegun(1965)]{abramowitz65a}
M.~Abramowitz and I.~A. Stegun.
\newblock \emph{Handbook of mathematical functions: with formulas, graphs, and
  mathematical tables}.
\newblock Dover Publications, New York, 1965.

\bibitem[Adams et~al.(2010)Adams, Ghahramani, and Jordan]{adams10}
R.~P. Adams, Z.~Ghahramani, and M.~I. Jordan.
\newblock Tree-structured stick breaking for hierarchical data.
\newblock In J.~Lafferty, C.~K.~I. Williams, J.~Shawe-Taylor, R.~S. Zemel, and
  A.~Culotta, editors, \emph{Advances in Neural Information Processing Systems
  23}, pages 19--27. MIT Press, 2010.

\bibitem[Dagpunar(1989)]{dagpunar89a}
J.~S. Dagpunar.
\newblock An easily implemented generalised inverse {G}aussian generator.
\newblock \emph{Communications in Statistics - Simulation and Computation},
  18\penalty0 (2):\penalty0 703--710, 1989.

\bibitem[Doucet et~al.(2001)Doucet, de~Freitas, Gordon, and Smith]{doucet01}
A.~Doucet, N.~de~Freitas, N.~Gordon, and A.~Smith.
\newblock \emph{Sequential Monte Carlo Methods in Practice}.
\newblock Springer, 2001.

\bibitem[Eltoft et~al.(2006)Eltoft, Kim, and Lee]{eltoft06a}
T.~Eltoft, T.~Kim, and T.-W. Lee.
\newblock On the multivariate {L}aplace distribution.
\newblock \emph{{IEEE} Signal Processing Letters}, 13\penalty0 (5):\penalty0
  300--303, 2006.

\bibitem[G\"{o}r\"{u}r and Teh(2009)]{gorur08}
D.~G\"{o}r\"{u}r and Y.~W. Teh.
\newblock An efficient sequential {M}onte {C}arlo algorithm for coalescent
  clustering.
\newblock In D.~Koller, D.~Schuurmans, Y.~Bengio, and L.~Bottou, editors,
  \emph{Advances in Neural Information Processing Systems 21}, pages 521--528.
  MIT Press, 2009.

\bibitem[Heller and Ghahramani(2005)]{heller05}
K.~A. Heller and Z.~Ghahramani.
\newblock Bayesian hierarchical clustering.
\newblock In \emph{Proceedings of the 22nd international conference on Machine
  learning}, pages 297--304. ACM, 2005.

\bibitem[Henao et~al.(2012)Henao, Thompson, Moseley, Ginsburg, Carin, and
  Lucas]{henao12a}
R.~Henao, J.~W. Thompson, M.~A. Moseley, G.~S. Ginsburg, L.~Carin, and J.~E.
  Lucas.
\newblock Hierarchical factor models for proteomics data.
\newblock In \emph{Computational Advances in Bio and Medical Sciences
  ({ICCABS}), 2012 {IEEE} 2nd International Conference on}, 2012.

\bibitem[Hubert and Arabie(1985)]{hubert85a}
L.~Hubert and P.~Arabie.
\newblock Comparing partitions.
\newblock \emph{Journal of Classification}, 2\penalty0 (1):\penalty0 193--218,
  1985.

\bibitem[J{\o}rgensen(1982)]{jorgensen82a}
Bent J{\o}rgensen.
\newblock \emph{Statistical properties of the generalized inverse Gaussian
  distribution}, volume~9 of \emph{Lecture notes in statistics}.
\newblock Springer-Verlag, 1982.

\bibitem[Kingman(1982{\natexlab{a}})]{kingman82}
J.~F.~C. Kingman.
\newblock The coalescent.
\newblock \emph{Stochastic processes and their applications}, 13\penalty0
  (3):\penalty0 235--248, 1982{\natexlab{a}}.

\bibitem[Kingman(1982{\natexlab{b}})]{kingman82a}
J.~F.~C. Kingman.
\newblock On the genealogy of large populations.
\newblock \emph{Journal of Applied Probability}, 19:\penalty0 27--43,
  1982{\natexlab{b}}.

\bibitem[Neal(2003{\natexlab{a}})]{neal03}
R.~Neal.
\newblock Density modeling and clustering using {D}irichlet diffusion trees.
\newblock In J.~Bernardo, M.~Bayarri, J.~Berger, A.~Dawid, D.~Heckerman,
  A.~Smith, and M.~West, editors, \emph{Bayesian Statistics 7}, pages 619--629.
  Oxford University Press, 2003{\natexlab{a}}.

\bibitem[Neal(2003{\natexlab{b}})]{neal03a}
R.~M. Neal.
\newblock Slice sampling.
\newblock \emph{Annals of Statistics}, 31\penalty0 (3):\penalty0 705--741,
  2003{\natexlab{b}}.

\bibitem[Pearl(1988)]{pearl88}
J.~Pearl.
\newblock \emph{Probabilistic Reasoning in Intelligent Systems: Networks of
  Plausible Inference}.
\newblock Morgan Kaufmann, 1988.

\bibitem[Rai and {Daume III}(2009)]{rai08}
P.~Rai and H.~{Daume III}.
\newblock The infinite hierarchical factor regression model.
\newblock In D.~Koller, D.~Schuurmans, Y.~Bengio, and L.~Bottou, editors,
  \emph{Advances in Neural Information Processing Systems 21}, pages
  1321--1328. MIT Press, 2009.

\bibitem[Rasmussen and Williams(2006)]{rasmussen06}
C.~E. Rasmussen and C.~K.~I. Williams.
\newblock \emph{Gaussian Processes for Machine Learning}.
\newblock The MIT Press, Cambridge, MA, 2006.

\bibitem[Teh et~al.(2008)Teh, {Daume III}, and Roy]{teh08}
Y.~W. Teh, H.~{Daume III}, and D.~Roy.
\newblock Bayesian agglomerative clustering with coalescents.
\newblock In J.~C. Platt, D.~Koller, Y.~Singer, and S.~T. Roweis, editors,
  \emph{Advances in Neural Information Processing Systems 20}, pages
  1473--1480. MIT Press, 2008.

\bibitem[Zhang et~al.(2011)Zhang, Dunson, and Carin]{zhang11a}
X.~Zhang, D.~Dunson, and L.~Carin.
\newblock Tree-structured infinite sparse factor model.
\newblock In L.~Getoor and T.~Scheffer, editors, \emph{Proceedings of the 28th
  International Conference on Machine Learning (ICML-11)}, pages 785--792,
  2011.

\end{thebibliography}
\end{document}